\documentclass[letterpaper, 10 pt, conference]{ieeeconf}  
\usepackage{amsmath}

\IEEEoverridecommandlockouts                              

\usepackage{amsmath, amssymb}
\usepackage{mathtools}
\usepackage{graphicx}
\usepackage{listings}
\usepackage{cite}
\usepackage{xcolor}
\usepackage[nolist,nohyperlinks]{acronym}

\usepackage{amsthm}
\usepackage{siunitx}
\usepackage{booktabs}
\usepackage{tabularray}
\usepackage{makecell}
\usepackage{multirow}
\usepackage{stfloats} 
\usepackage{tikz}
\usetikzlibrary{arrows.meta, calc, patterns, decorations.pathmorphing}

\usepackage[T1]{fontenc}      
\usepackage{pgfplots}
\usepgfplotslibrary{groupplots}
\pgfplotsset{compat=1.18}

\usepackage{hyperref}
\usepackage{pgfplotstable}

\newcommand{\prel}{\mathbf{p}_{\mathrm{rel}}}
\newcommand{\prelhat}{\hat{\mathbf{p}}_{\mathrm{rel}}}
\newcommand{\vrel}{\mathbf{v}_{\mathrm{rel}}}
\newcommand{\vrelhat}{\hat{\mathbf{v}}_{\mathrm{rel}}}
\newcommand{\norm}[1]{\left\lVert #1 \right\rVert}

\DeclareMathOperator{\VO}{\mathrm{VO}}
\DeclareMathOperator{\UVO}{\mathrm{UVO}}
\DeclareMathOperator{\DCPA}{\mathtt{DCPA}}
\DeclareMathOperator{\TCPA}{\mathtt{TCPA}}

\DeclareMathOperator*{\argmin}{arg\,min} %
\DeclareMathOperator*{\argmax}{arg\,max} %

\DeclareMathOperator{\CPA}{\mathtt{CPA}}

\renewenvironment{proof}[1][\proofname]{\par\itshape\noindent\textit{#1.}\enspace}{\par}

\newtheorem{definition}{Definition}
\newtheorem{theorem}{Theorem}
\newtheorem{lemma}{Lemma}

\newtheorem{assumption}{Assumption}
\newtheorem{corollary}{Corollary}

\theoremstyle{definition} 

\makeatletter
\def\@opargbegintheorem#1#2#3{\trivlist
  \item[]{#1\ #2\ (\itshape #3\bfseries): }\itshape}
\makeatother

\title{\LARGE \bf
Expressiveness, Equivalence, and Uncertainty in Velocity Obstacles and Closest Point of Approach Metrics
}

\author{Elizabeth Dietrich$^{*, 1}$, Liam M. Imagawa$^{*, 1}$, Hanna Krasowski$^{1}$, \\Aurora Haraldsen$^{2}$, Murat Arcak$^{1}$, and Kristin Y. Pettersen$^{2}$ 
\thanks{$^{*}$Equal contribution.}
\thanks{$^{1}$E. Dietrich, L. M. Imagawa, H. Krasowski, and M. Arcak are with the University of California, Berkeley, USA. Email: \texttt{\{eadietri, liamimagawa, krasowski, arcak\}@berkeley.edu}.}
\thanks{$^{2}$A. Haraldsen and K.\,Y. Pettersen are with the Norwegian University of Science and Technology, Trondheim, Norway. Email: \texttt{\{Aurora.Haraldsen, Kristin.Y.Pettersen\}@ntnu.no}.}
\thanks{
This work was funded in part by the NSF under grant CNS-2111688.
E. Dietrich was also supported by an NSF Graduate Research Fellowship. A. Haraldsen and K.Y. Pettersen were partially supported by VISTA – a basic research program in collaboration between The Norwegian Academy of Science and Letters, and Equinor.}
\vspace{-3mm}
}

\begin{document}

\begin{acronym}
    \acro{TCPA}{Time to Closest Point of Approach}
    \acro{DCPA}{Distance to Closest Point of Approach}
    \acro{CPA}{Closest Point of Approach}
    \acro{VO}{Velocity Obstacle}
    \acro{UVO}{Uncertainty-Aware Velocity Obstacle}
    \acro{COLREG}{Convention on the International Regulations for Preventing Collisions at Sea }
\end{acronym}

\maketitle
\thispagestyle{empty}
\pagestyle{empty}

\begin{abstract}
Time to Closest Point of Approach (TCPA), Distance to Closest Point of Approach (DCPA), and Velocity Obstacles (VOs), are widely used to assess and mitigate collision risk in autonomous navigation, yet their relationship and behavior under uncertainty remain largely unexplored.
Assuming perfect state information, we establish a relationship between these representations over finite and infinite prediction horizons and derive conditions under which they provide equivalent characterizations of collision risk. 
Under bounded uncertainty, we extend the Closest Point of Approach (CPA) metrics and VO to convex relative-state sets.
We show that in this setting, independently computed TCPA and DCPA bounds lose the joint relationship required for VO membership, while uncertainty-aware VOs preserve this relationship through a set-valued representation of collision-inducing velocities.
\end{abstract}

\vspace{-1mm}
\section{Introduction}\label{sec:intro}
\vspace{-1mm}
Collision risk metrics are essential for monitoring and control of autonomous navigation, where potential collisions must be identified and mitigated in a timely manner \cite{marvel2014, vesentini2024}. 
In unstructured environments, including maritime~\cite{huang2018, huang2020}, aviation \cite{errico}, and space \cite{Klinkrad2006} domains, risk assessment for navigation cannot rely on predefined lanes or other environmental structures. Instead, collision risk must be determined from observations of the relative kinematics and embodiment of nearby agents. We focus on criticality metrics that capture the spatial and temporal proximity of potential collisions across domains, in particular the \ac{CPA} metrics \cite{silveira2021, radanovic2018} and the \ac{VO}~\cite{fiorini1998, wilkie2009, kuwata2011, zhang2021, vanderberg2011, kufoalor2018}.  

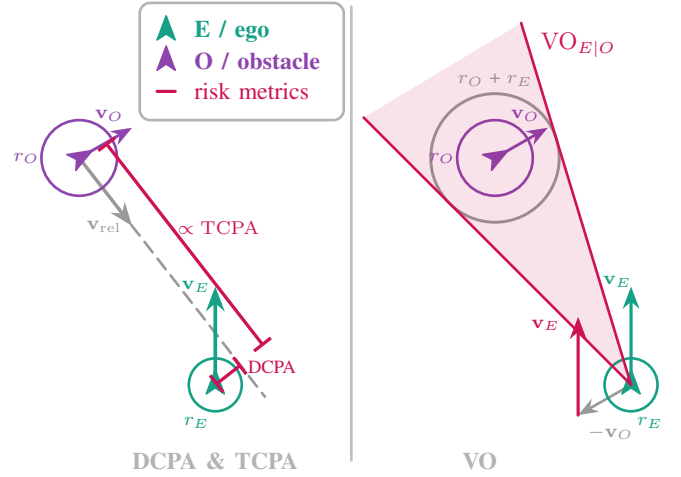
\begin{figure}
    \centering              

\definecolor{colorA}{HTML}{16A085} 
\definecolor{colorB}{HTML}{8E44AD} 
\definecolor{colorRisk}{HTML}{CE1256} 
\definecolor{colorAgent}{RGB}{40, 40, 40}    

\newcommand{\agentsymbol}[5]{%
  \begin{scope}[shift={(#1,#2)}, rotate=#3]
    \fill[#4, draw=#5, line width=0.6pt]
      (0, 0.20) -- (-0.12, -0.12) -- (0, -0.04) -- (0.12, -0.12) -- cycle;
  \end{scope}
}

\newcommand{\agentsymbolcircle}[7]{%
  \draw[#6, line width=1.0pt] (#1,#2) circle (#7);
  \agentsymbol{#1}{#2}{#3}{#4}{#5}
}

\begin{tikzpicture}[
  >=Stealth,
  font=\small,
  line cap=round,
  line join=round
]

\def\Lx{2.0}   \def\Ly{-3.5}    
\def\Rx{7.5} \def\Ry{-3.5}    

\def\Ax{0}    \def\Ay{0}
\def\Bx{-1.8} \def\By{3.0}

\def\vAx{0.0}  \def\vAy{1.3}    
\def\vBx{0.7}  \def\vBy{0.4}    

\pgfmathsetmacro{\nvAx}{-\vAx}
\pgfmathsetmacro{\nvAy}{-\vAy}
\pgfmathsetmacro{\nvBx}{-\vBx}
\pgfmathsetmacro{\nvBy}{-\vBy}

\pgfmathsetmacro{\vRelx}{\vBx - \vAx}
\pgfmathsetmacro{\vRely}{\vBy - \vAy}
\pgfmathsetmacro{\vBangle}{atan2(\vBy, \vBx) - 90}

\def\rA{0.35}
\def\rB{0.5}
\pgfmathsetmacro{\rSum}{\rA + \rB}   

\draw[gray!50, line width=1pt] (3.8, -4.5) -- (3.8, 1.5);


\begin{scope}[shift={(1.0, 0.0)}]   
  \draw[gray!60, rounded corners=4pt, line width=1pt]
    (0, 0) rectangle (2.7, 1.55);

  \agentsymbol{0.35}{1.18}{0}{colorA}{colorA}
  \node[right, colorA] at (0.6, 1.18) {\textbf{E / ego}};

  \agentsymbol{0.35}{0.78}{0}{colorB}{colorB}
  \node[right, colorB] at (0.6, 0.78) {\textbf{O / obstacle}};

  \draw[colorRisk, line width=1.2pt] (0.23, 0.38) -- (0.47, 0.38);
  \node[right, colorRisk] at (0.6, 0.38) {risk metrics};
\end{scope}








\begin{scope}[shift={(\Lx, \Ly)}]

  \pgfmathsetmacro{\vRelNorm}{sqrt(\vRelx*\vRelx + \vRely*\vRely)}
    \pgfmathsetmacro{\vRelUx}{\vRelx / \vRelNorm}
    \pgfmathsetmacro{\vRelUy}{\vRely / \vRelNorm}
    \def\dashLen{4}
    \draw[black!40, line width=1.0pt, dashed, dash pattern=on 5pt off 3pt]
    (\Bx, \By) -- ({\Bx + \dashLen*\vRelUx}, {\By + \dashLen*\vRelUy});


  \draw[->, colorB, line width=1.1pt]
    (\Bx, \By) -- ++(\vBx, \vBy)
    node[above, xshift=-0.3cm, yshift=-0.05cm,  colorB, font=\scriptsize] {$\mathbf{v}_O$};

  \draw[->, black!40, line width=1.3pt]
    (\Bx, \By) -- ++(\vRelx, \vRely)
    node[left, black!40, font=\scriptsize] {$\mathbf{v}_\mathrm{rel}$};

  \agentsymbolcircle{\Bx}{\By}{\vBangle}{colorB}{colorB}{colorB}{\rB}
  \node[colorB, font=\scriptsize] at ({\Bx - \rB - 0.2}, \By) {$r_O$};

  \agentsymbolcircle{\Ax}{\Ay}{0}{colorA}{colorA}{colorA}{\rA}
  \node[colorA, font=\scriptsize] at (-0.25, -0.5) {$r_E$};

  \draw[->, colorA, line width=1.4pt]
    (\Ax, \Ay) -- ++(\vAx, \vAy)
    node[left, xshift=2, colorA, font=\scriptsize] {$\mathbf{v}_E$};

  \draw[colorRisk, {Bar[]}-{Bar[]}, line width=1.25pt]
    (0, 0.0) -- ++(-0.38*\vRely, 0.38*\vRelx)
    node[midway, right, xshift=3.5pt, yshift=3pt, colorRisk, font=\scriptsize] {DCPA};

    \draw[colorRisk, {Bar[]}-{Bar[]},line width=1.25pt] ({\Bx+0.35},{\By+0.2}) --++ ({3.4*\vRelUx}, {3.4*\vRelUy}) node[midway, above, xshift=0.45cm, colorRisk, font=\scriptsize] {$\propto \mathrm{TCPA}$};

  \node[gray!60, font=\large\bfseries] at (0.0, -1.0) {{\small DCPA \& TCPA}};

\end{scope}

\begin{scope}[shift={(\Rx, \Ry)}, scale=1]

  \pgfmathsetmacro{\dist}{sqrt(\Bx*\Bx + \By*\By)}
  \pgfmathsetmacro{\halfAngle}{asin(\rSum/\dist)}
  \pgfmathsetmacro{\centerAngle}{atan2(\By, \Bx)}
  \pgfmathsetmacro{\angLeft}{\centerAngle + \halfAngle}
  \pgfmathsetmacro{\angRight}{\centerAngle - \halfAngle}
  \def\coneLen{5.0}

  \draw[black!40, line width=1.1pt] (\Bx, \By) circle (\rSum);
  \node[black!40, font=\scriptsize] at (\Bx, {\By + \rSum + 0.18}) {$r_O$ {+} $r_E$};

  \agentsymbolcircle{\Bx}{\By}{\vBangle}{colorB}{colorB}{colorB}{\rB}
  \node[colorB, font=\scriptsize] at ({\Bx - \rB - 0.2}, \By) {$r_O$};

  \draw[->, colorB, line width=1.1pt]
    (\Bx, \By) -- ++(\vBx, \vBy)
    node[above, xshift=-0.3cm, yshift=-0.05cm, colorB, font=\scriptsize] {$\mathbf{v}_O$};

  \draw[->, black!40, line width=1.0pt]
    (\Ax, \Ay) -- ++(\nvBx, \nvBy)
    node[below right, black!40, font=\scriptsize] {$-\mathbf{v}_O$};

  \agentsymbolcircle{\Ax}{\Ay}{0}{colorA}{colorA}{colorA}{\rA}
  \node[colorA, font=\scriptsize] at (0.25, -0.5) {$r_E$};

  \draw[->, colorA, line width=1.2pt]
    (\Ax, \Ay) -- ++(\vAx, \vAy)
    node[left, colorA, xshift=4, yshift=2, font=\scriptsize] {$\mathbf{v}_E$};

  \fill[colorRisk, opacity=0.12]
    (0,0)
    -- ({\coneLen*cos(\angLeft)},  {\coneLen*sin(\angLeft)})
    -- ({\coneLen*cos(\angRight)}, {\coneLen*sin(\angRight)})
    -- cycle;

  \draw[colorRisk, line width=1.0pt]
    (0,0) -- ({\coneLen*cos(\angLeft)},  {\coneLen*sin(\angLeft)});
  \draw[colorRisk, line width=1.0pt]
    (0,0) -- ({\coneLen*cos(\angRight)}, {\coneLen*sin(\angRight)});

  \node[colorRisk, font=\small] at (-0.7, 4.5) {$\mathrm{VO}_{E|O}$};
  
  \draw[->, colorRisk, line width=1.2pt]
    ($ (\Ax, \Ay) - (\vBx, \vBy) $) -- ++(\vAx, \vAy)
    node[left, colorRisk, xshift=-3, yshift=-3, font=\scriptsize] {$\mathbf{v}_E$};

  \node[gray!60, font=\large\bfseries] at (-2.0, -1.0) {{\small VO}};

\end{scope}

\pgfresetboundingbox
\path (-0.7, -4.0) rectangle (11.5, 1.7);

\end{tikzpicture}
        \caption{Overview of risk metrics. Left: DCPA is the perpendicular distance from the position $\mathbf{p}_E$ of agent $E$, to the relative trajectory (gray dashed line). The relative trajectory passes through $\mathbf{p}_O$ and is aligned with $\vrel$. 
        TCPA is the time until the agents reach the \ac{CPA} and is proportional to the distance between $\mathbf{p}_O$ and the \ac{CPA}.  Right: The \ac{VO} represented as a collision cone in relative-velocity space. The cone is bounded by tangent rays extending from the ego reference point $\mathbf{p}_E$ to the circular collision region centered at $\mathbf{p}_O$, with radius $r_E+r_O$. }
        \vspace{-5mm}
    \label{fig:risk_metrics_SOTA}
\end{figure}

\ac{CPA} metrics assess collision risk by considering the minimum separation between agents over a prediction horizon.
Specifically, the \ac{DCPA} and the \ac{TCPA} provide scalar measures of collision risk for a given relative velocity and position, yielding a compact trajectory risk measure \cite{kuwata2011, kufoalor2018, zhang2021, gai2024, wilkie2009}. To support their use in practical evasive maneuvering, \ac{CPA} metrics have been extended to incorporate domain-specific information, including traffic rules, agent dimensions, and other movement information~\cite{huang2020, Ha21,Zhao22,Zhou2026}. 

However, from a control perspective, collision avoidance requires evaluating collision risk over all feasible velocities across the prediction horizon. The \ac{VO} provides a set-valued 
representation that partitions the velocity space into collision-inducing and collision-free velocities, enabling direct identification of evasive control actions.
The \ac{VO} has been applied in a variety of motion planning and control contexts, both as a collision-avoidance mechanism and as a component of broader control frameworks, including integration with control barrier functions \cite{roncero2025, haraldsen2024}, model predictive control \cite{cheng2017, piccinelli2023}, and reinforcement learning \cite{Xie2023}. 

Despite their common objective of characterizing collision risk, the relationship between \ac{CPA} metrics and \ac{VO} representations remains largely unexplored. 
Existing work relates \ac{VO} membership to the \ac{DCPA} condition \cite{huang2018}, but does not account for the \ac{TCPA} condition. When the \ac{TCPA} is negative, the point of closest approach occurs in the past; therefore, the \ac{DCPA} condition alone does not indicate a future collision. We show that the \ac{DCPA} condition is necessary but not sufficient for \ac{VO} membership, and formally establish the additional \ac{TCPA} condition required for equivalence. 
Furthermore, the geometric equivalence between the \ac{VO} boundary and the \ac{DCPA} condition for non-negative \ac{TCPA}, established in \cite[Prop. 2.1]{kufoalor2018}, is recovered as a special case of our results. 
Our main contributions include:
    \begin{itemize}
        \item We derive necessary and sufficient conditions for \ac{VO} membership in terms of \ac{TCPA} and \ac{DCPA} over finite and infinite prediction horizons.
        \item We formulate an uncertainty-aware \ac{VO} for general convex uncertainty sets, preserving its exact geometry.
        \item We demonstrate that under bounded uncertainty, independently computed \ac{TCPA} and \ac{DCPA} bounds lose the joint relationship required for \ac{VO} membership.
     \end{itemize}

The remainder of the paper is organized as follows.
Sec.~\ref{sec:prelim} introduces the considered metrics. Sec.~\ref{sec:theory} derives necessary and sufficient conditions for \ac{VO} membership using \ac{CPA} metrics. Sec.~\ref{sec:uncertainty} formulates an uncertainty-aware \ac{VO}. Finally, Sec.~\ref{sec:experiments} evaluates the metrics under bounded uncertainty.

\section{Preliminaries}\label{sec:prelim}
We consider collision risk assessment in unstructured environments, where structural constraints on agent motion may not be available. For prediction, we therefore assume that agents maintain constant velocities $\mathbf{v}_E, \mathbf{v}_O$ over the prediction horizon, which may be finite or unbounded. 

\subsection{Notation}
We denote vectors in bold, sets by calligraphic letters, and functions in typewriter font. For $\mathbf{u}, \mathbf{v} \in \mathbb{R}^2$, the two-dimensional cross product is defined as $\mathbf{u} \times \mathbf{v} = u_x v_y - u_y v_x$. For subsets $\mathcal{X}, \mathcal{Y}$ of the Euclidean space, their Minkowski sum is defined as $\mathcal{X} \oplus \mathcal{Y}= \{x + y \mid x \in \mathcal{X}, y \in \mathcal{Y}\}$. All norms are taken to be the Euclidean norm. 
We consider ego and obstacle agents, denoted by subscripts $E$ and $O$, respectively.
Their positions and velocities are denoted by $\mathbf{p}_E, \mathbf{p}_O\in \mathbb{R}^2$ and $\mathbf{v}_E, \mathbf{v}_O \in \mathbb{R}^2$, respectively.
The relative position $\prel$ and velocity $\vrel$ of the obstacle with respect to the ego are defined as $\prel = \mathbf{p}_O - \mathbf{p}_E$ and $\vrel = \mathbf{v}_O - \mathbf{v}_E$.

\subsection{Scalar Risk Metrics}
Two widely used \ac{CPA} metrics are \ac{DCPA}, the minimum separation distance between two agents, and \ac{TCPA}, the time associated with this closest approach~\cite{kuwata2011, kufoalor2018, zhang2021, gai2024, wilkie2009}. 
Their scalar form enables computationally efficient assessment and facilitates the specification of risk thresholds. 
Given relative position $\prel$ and relative velocity $\vrel$, the \ac{CPA} metric functions are defined as follows:
\begin{align}
    &\DCPA(\prel, \vrel) = 
    \begin{cases}
        \norm{\prel} &\text{if } \norm{\vrel} = 0 \\
        \frac{|\vrel \times \prel|}{\norm{\vrel}} &\text{otherwise,}
    \end{cases}\label{eq:dcpa} \\
    &\TCPA(\prel, \vrel)  = 
    \begin{cases}
        0 &\text{if } \norm{\vrel} = 0 \\
        - \frac{\prel \cdot \vrel}{\norm{\vrel}^2} &\text{otherwise.}
    \end{cases}\label{eq:tcpa}
\end{align}
For $\vrel \neq \mathbf{0}$, 
we obtain the equivalent expression:
\begin{equation}\label{eq:dcpa-dot}
    \DCPA^2(\prel, \vrel) = \norm{\prel}^2 - \frac{(\prel \cdot \vrel)^2}{\norm{\vrel}^2}.
\end{equation}

\ac{TCPA} captures the temporal aspect of an encounter, while \ac{DCPA} measures spatial proximity without indicating the side of closest approach. When this distinction is relevant, we define a signed
$\DCPA_s: \mathbb{R}^2 \times (\mathbb{R}^2 \setminus \{\mathbf{0}\}) \to \mathbb{R}$ as follows:
\begin{equation}
    \DCPA_s(\prel, \vrel) =\frac{\vrel \times \prel}{\norm{\vrel}}.
\end{equation}

\subsection{Velocity Obstacle}

Consider ego and obstacle agents represented by convex sets $\mathcal{E}, \mathcal{O} \subseteq \mathbb{R}^2$, where their occupied space in the absolute position frame is defined by ${\mathbf{p}_E} \oplus \mathcal{E}$ and ${\mathbf{p}_O} \oplus \mathcal{O}$, respectively.
The reflection $-\mathcal{O} = \{-\mathbf{o} \mid \mathbf{o} \in \mathcal{O}\}$ and the Minkowski sum $\mathcal{E} \oplus (-\mathcal{O})$ characterize the set of relative positions for which the ego and obstacle overlap, indicating a collision. In contrast to the \ac{CPA} metrics, the \ac{VO} retains spatial information needed to characterize collision geometry.
Specifically, the \ac{VO} is the set of ego velocities $\mathbf{v}_E$ that result in a collision at some time during the prediction horizon:
\begin{equation} \label{eq:vo}
    \VO_{E\mid O} = \{ \mathbf{v}_E \mid \exists t \geq 0 ,  \prel + \vrel t \in \mathcal{E} \oplus (-\mathcal{O}) \}.
\end{equation}
Equivalently, the \ac{VO} can be expressed as a collision cone, comprising the set of relative velocities $\vrel$ whose rays from $\prel$ intersect $\mathcal{E} \oplus (-\mathcal{O})$.

\subsection{Problem Statement}
We consider collision risk arising from the relative states $\prel, \vrel \in \mathbb{R}^2$ of two agents with fixed geometry, assuming perfect observation and constant relative velocity.
Under these assumptions, the \ac{CPA} metrics and \ac{VO} provide complementary characterizations of collision risk.
We aim to establish a formal relationship between \ac{TCPA}, \ac{DCPA}, and the collision-inducing velocities represented by the \ac{VO}, considering both infinite and finite time horizons.
We then relax the assumption of perfect observation by introducing bounded uncertainty in the relative states. We characterize the effect of this uncertainty on the collision risk metrics, including a computational complexity evaluation.

\section{Theoretical Analysis}\label{sec:theory}

In this section, we derive necessary and sufficient conditions for \ac{VO} membership in terms of \ac{TCPA} and \ac{DCPA} in the absence of uncertainty. 
For brevity, we omit function arguments of \ac{TCPA} and \ac{DCPA}.  

\begin{assumption}[Perfect Observations]\label{a2}
    The current positions $\mathbf{p}_E, \mathbf{p}_O$ and velocities $\mathbf{v}_E, \mathbf{v}_O$ of the two agents are observed without uncertainty.
\end{assumption}
\noindent Additionally, we only consider circular agents, for which the collision region is characterized by a single radius $R$.

\begin{definition}[Circular \acp{VO}]\label{def:circles}
The ego and obstacle agent geometries are modeled as circular sets with radii $r_E$ and $r_O$, i.e., $\mathcal{E} = \{ \mathbf{x} \in \mathbb{R}^2 \,|\, \|  \mathbf{x} \| \leq r_E\}$ and $\mathcal{O} = \{ \mathbf{x} \in \mathbb{R}^2 \,|\, \| \mathbf{x} \| \leq r_O\}$. The resulting collision region, $\mathcal{E} \oplus (-\mathcal{O})$, is a circle of radius $R = r_E + r_O$, and \eqref{eq:vo} reduces to:
\begin{equation}\label{eq:vo-circ}
    \VO_{E\mid O} = \{\mathbf{v}_E \in \mathbb{R}^2 \mid \exists t \geq 0, \norm{\prel + \vrel t} \leq R \}.
\end{equation}
\end{definition}

The \ac{VO} condition in \eqref{eq:vo-circ} depends on the separation distance between agents along their relative trajectory, defined as $d(t, \prel, \vrel)\triangleq\norm{\prel + \vrel t}, \forall t \in \mathbb{R}$.
This distance can be decomposed in terms of \ac{TCPA} and \ac{DCPA} as follows.
\begin{lemma}\label{lemma:sq-dist}
    Given $\vrel \neq \mathbf{0}$, for every $t \in \mathbb{R}$,
    \begin{equation}\label{eq:separation_identity}
        d^2(t, \prel, \vrel)=\norm{\vrel}^2 \left(t - \TCPA \right)^2 + \DCPA^2.
    \end{equation}
\begin{proof} 
\vspace{-4mm}
    \begin{align}
        d^2(t, \prel, \vrel) 
        & = \norm{\vrel}^2\left(t + \frac{\prel\cdot\vrel}{\norm{\vrel}^2}\right)^2 \notag \\
        & \quad \quad \quad \quad \quad - \frac{(\prel\cdot\vrel)^2}{\norm{\vrel}^2} + \norm{\prel}^2. \notag
    \end{align}
    Combining with \eqref{eq:tcpa} and \eqref{eq:dcpa-dot} yields \eqref{eq:separation_identity}. \qed
\end{proof} 
\end{lemma} 
\noindent For zero relative velocity, this separation distance is constant. Therefore, $\vrel = \mathbf{0}$ yields $ d^2(t, \prel, \vrel)= \DCPA^2.$

To determine whether a relative trajectory violates the safety radius $R$, it suffices to evaluate the minimum relative distance.
For each pair $(\prel, \vrel)$, $\TCPA$ is a global minimizer of $d(t, \prel, \vrel)$ over $t \in \mathbb{R}$, i.e., $\argmin_{t\in \mathbb{R}} d(t, \prel, \vrel) = \TCPA$.
However, predictive safety is typically only concerned with $t\geq 0$. Therefore, the following corollaries partition this domain according to the sign of the minimizer, $\TCPA$, to characterize the minimum distance attained for $t\geq 0$.

\begin{corollary}\label{cor:tcpagreaterthan0}
If $\TCPA > 0$, the minimizer of $d(t, \prel, \vrel)$ over $t\geq 0$ is $t=\TCPA > 0$, with minimum value $d(\TCPA, \prel, \vrel) = \DCPA$.
\begin{proof}
The global minimum of $d(t, \prel, \vrel)$ is at $t = \TCPA$.
Because $\TCPA > 0$, 
\begin{equation}
    \min_{t \geq 0} d(t, \prel, \vrel) = d(\TCPA, \prel, \vrel) = \DCPA. \notag \qed
\end{equation}
\end{proof}
\end{corollary}

\begin{corollary}\label{cor:tcpalessthan0}
If $\TCPA \leq 0$, the minimizer of $d(t, \prel, \vrel)$ over $t\geq 0$ is attained at the boundary $t=0$, with minimum value $d(0, \prel, \vrel) = \| \prel \|$.
\begin{proof}
The global minimum of $d(t, \prel, \vrel)$ is at $t = \TCPA$.
Because $\TCPA \leq 0$, the function $d$ is monotonically increasing with respect to $t \geq 0$. Therefore, for any $t \geq 0$, $d^2(t, \prel, \vrel) \geq d^2(0, \prel, \vrel) = \| \prel \|^2$, which implies
\begin{equation}
\min_{t \geq 0} d(t, \prel, \vrel) = d(0, \prel, \vrel) = \| \prel \|. \notag \qed
\end{equation}
\end{proof}
\end{corollary}

Depending on the sign of \ac{TCPA}, the minimum separation distance over $t\geq 0$ is determined by either \ac{DCPA} or the current separation distance. Based on this relationship, we show in the following that the infinite-horizon \ac{VO} can be characterized in terms of the \ac{CPA} metrics. 

\begin{theorem}[Infinite-Horizon Equivalence]\label{thm:diskvo-cpa}
Given Definition~\ref{def:circles} and Assumption~\ref{a2},
\begin{align}\label{eq:infinite_time_relationship}
        &\exists t \geq 0 , \norm{\prel + \vrel t} \leq R \iff \\
        &(\TCPA > 0 \wedge \DCPA \leq R) \vee
    (\TCPA \leq 0 \wedge \norm{\prel} \leq R).\notag
\end{align}
\begin{proof} We refer to the left-hand side of \eqref{eq:infinite_time_relationship} as the \ac{VO} condition, and the right hand side as \ac{CPA} metrics.\\
\textbf{\ac{VO} Condition $\implies$ \ac{CPA} Metrics.}\\
Suppose 
 $   \exists t \geq 0$ 
 \text{ s.t. } 
 $\| \prel + \vrel t\| =d(t, \prel, \vrel) \leq R. \notag $
Then, the minimum of $d(t, \prel, \vrel)$ over $t\geq 0$ is bounded by $R$. 
If $\TCPA \leq 0$, the minimum value of $d(t, \prel, \vrel)$ is $\| \prel \|$, by Corollary \ref{cor:tcpalessthan0}. Hence,  $(\TCPA \leq 0 \wedge \| \prel \| \leq R)$.
If $\TCPA>0$, the minimum value of $d(t, \prel, \vrel)$ is $\DCPA$, by Corollary \ref{cor:tcpagreaterthan0}. Hence, $(\TCPA > 0 \wedge \DCPA \leq R)$.

\noindent \textbf{\ac{CPA} Metrics $\implies$ \ac{VO} Condition.}\\
Suppose 
 \(   (\TCPA > 0 \wedge \DCPA \leq R) \vee (\TCPA \leq 0 \wedge \| \prel \| \leq R) \notag
 \)
holds.
If $\TCPA > 0$ and $\DCPA \leq R$, choose $t' = \TCPA > 0$.
Since $d(\TCPA, \prel, \vrel) = \DCPA \leq R$, it follows that $\exists t \geq 0$ s.t. $\| \prel + \vrel t \| \leq R.$
If $\TCPA \leq 0$ and $\| \prel \| \leq R$, choose $t'=0$. Since $ d(0, \prel, \vrel) = \| \prel \| \leq R$, it follows that $\exists t \geq 0$ s.t. $\| \prel + \vrel t \| \leq R.$ 
\qed
\end{proof}
\end{theorem}

The condition $\|\prel\| \leq R$ indicates that the agents' geometries already overlap. 
For initially separated agents, $\|\prel\| > R$, the following provides a more practical characterization of the \ac{VO} for future collisions.

\begin{corollary}\label{cor:vo-cpa}
    If $\norm{\prel} > R$, \eqref{eq:infinite_time_relationship} reduces to
    \begin{equation}\label{eq:vointcpadcpa}
        \VO_{E\mid O} = \{\, \mathbf{v}_E\in \mathbb{R}^2 \mid \TCPA > 0 \wedge \DCPA \le R \,\}.
    \end{equation}
\end{corollary}

The infinite-horizon \ac{VO} captures relative trajectories that violate the safety radius $R$ at some future point in time and states that violate $R$ at $t=0$. To avoid premature reactions to conflicts that may occur only far in the future, finite-horizon \acp{VO} are often used to restrict the time domain to a finite prediction horizon $t \in [0, \tau]$, where $\tau > 0$. This yields
\begin{equation}
    \VO_{E\mid O}(\tau) = \{\mathbf{v}_E \in \mathbb{R}^2 \mid \exists t \in [0, \tau], \norm{\prel + \vrel t} \leq R\}. \notag
\end{equation}
We now characterize finite-horizon \ac{VO} membership in terms of the \ac{CPA} metrics. 

\begin{table*}[t]
\caption{Support functions, support vectors, and computational complexity for common convex sets in $\mathbb{R}^2$.}
\label{tab:support}
\centering
\begin{tabular}{lccc}
\hline 
Set & $h_\mathcal{S}(\mathbf{d})$ & $v_\mathcal{S}(\mathbf{d})$ & Complexity \\
\hline
Unit Ball : $\{ \mathbf{x} \in \mathbb{R}^2 :\|\mathbf{x}\| \leq 1\}$ & $\|\mathbf{d}\|$
& $\frac{\mathbf{d}}{\|\mathbf{d}\|}$ & $\mathcal{O}(1)$ \\
Circle : $\{ \mathbf{x} \in \mathbb{R}^2 :\|\mathbf{x}-\mathbf{c}\| \leq r\}$ & $\mathbf{d}^\top \mathbf{c} + r\norm{\mathbf{d}}$
& $\mathbf{c} + r\frac{\mathbf{d}}{\norm{\mathbf{d}}}$ & $\mathcal{O}(1)$ \\
Ellipse : $\{ \mathbf{x} \in \mathbb{R}^2 :(\mathbf{x}-\mathbf{c})^\top Q^{-1}(\mathbf{x}-\mathbf{c}) \leq r\}$ & $\mathbf{d}^\top \mathbf{c} + \sqrt{\mathbf{d}^\top(rQ)\mathbf{d}}$
& $\mathbf{c} + \frac{(rQ)\mathbf{d}}{\sqrt{\mathbf{d}^\top(rQ)\mathbf{d}}}$ & $\mathcal{O}(1)$ \\
Polygon : $\{ \mathbf{x} \in \mathbb{R}^2 : H\mathbf{x} \leq \mathbf{b}\}$ & $\max_{\mathbf{x}\in \mathbb{R}^2} \mathbf{d}^\top\mathbf{x}, \text{ s.t. } H\mathbf{x} \leq \mathbf{b}$
& $\argmax_{\mathbf{x}\in \mathbb{R}^2} \mathbf{d}^\top\mathbf{x}, \text{ s.t. } H\mathbf{x} \leq \mathbf{b}$ & $\mathcal{O}(m)$\vspace{1mm} \\
\hline
\end{tabular}
\vspace{-3mm}
\end{table*}

\begin{theorem}[Finite-Horizon CPA Characterization]\label{thm:diskfhvo-cpa}
Given Definition~\ref{def:circles} and Assumption~\ref{a2},
\begin{align}
    &\exists t \in [0, \tau], \norm{\prel + \vrel t} \leq R \iff \notag \\
    &(\TCPA \leq 0 \wedge \norm{\prel} \leq R) \vee (0 < \TCPA \leq \tau \wedge \DCPA \leq R) \vee \notag \\
    & \hspace{38mm}(\TCPA > \tau \wedge \norm{\prel + \vrel \tau} \leq R) \notag
\end{align}
\begin{proof}\\
\noindent \textbf{\ac{VO} Condition $\implies$ \ac{CPA} Metrics.}\\
Suppose 
$\exists t \in [0, \tau]$ s.t. $\norm{\prel + \vrel t} =d(t, \prel, \vrel) \leq R.$ 
Then, the minimum of $d(t, \prel, \vrel)$ over $[0, \tau]$ is bounded by $R$. 
If $\TCPA \leq 0$, the minimum value of $d(t, \prel, \vrel)$ is $\| \prel \|$, by Corollary \ref{cor:tcpalessthan0}. Hence,  $(\TCPA \leq 0 \wedge \| \prel \| \leq R)$.
If $\TCPA > 0$ and $\TCPA \leq \tau$, the minimum value of $d(t, \prel, \vrel)$ is $\DCPA$, since the minimizer in Corollary \ref{cor:tcpagreaterthan0} lies within $[0, \tau]$. Hence, $(0 < \TCPA \leq \tau \wedge \DCPA \leq R)$.
If $\TCPA > \tau$, the minimizer of $d^2(t, \prel, \vrel)$ lies strictly to the right of $[0, \tau]$. Since $d^2(t, \prel, \vrel)$ is convex in $t$, it is monotonically decreasing $\forall t \leq \TCPA$. Therefore, the minimum on $[0, \tau]$ is attained at the upper boundary $t=\tau$:
\begin{align}
    \min_{t \in [0, \tau]} d^2(t, \prel, \vrel) 
    &= \| \prel + \vrel \tau\|^2 \leq R^2, \notag
\end{align}
which implies $(\TCPA >\tau \wedge \| \prel + \vrel \tau\| \leq R)$. 

\noindent \textbf{\ac{CPA} Metrics $\implies$ \ac{VO} Condition.}\\
Suppose $ (\TCPA \leq 0 \wedge \norm{\prel} \leq R) \vee (0 < \TCPA \leq \tau \wedge \DCPA \leq R) \vee 
    (\TCPA > \tau \wedge \norm{\prel + \vrel \tau} \leq R)$ holds.
If $\TCPA \leq 0$ and $\| \prel \| \leq R$, choose $t'=0$. Since $ d(0, \prel, \vrel) = \| \prel \| \leq R$, it follows that $\exists t \in [0,\tau] $ s.t. $ \| \prel + \vrel t \| \leq R.$ 
If $(0 < \TCPA \leq \tau \wedge \DCPA \leq R)$, choose $t' = \TCPA \in [0, \tau]$. Since $d(\TCPA, \prel, \vrel) = \DCPA \leq R$, it follows that $\exists t \in [0, \tau] $ s.t. $ \| \prel + \vrel t \| \leq R.$
If $(\TCPA > \tau \wedge \norm{\prel + \vrel \tau} \leq R)$, $\exists t \in [0, \tau] \text{ s.t. } \| \prel + \vrel t \| \leq R.$
\qed
\end{proof}
\end{theorem}

\section{Collision Risk Metrics under Uncertainty}\label{sec:uncertainty}
In practice, relative states are obtained through estimation and therefore subject to noise and uncertainty. 
In this section, we relax Assumption~\ref{a2} and extend the \ac{CPA} metrics and \ac{VO} to relative positions and velocities represented 
as convex sets.

\subsection{\ac{UVO}}
We assume bounded uncertainty of $\mathbf{p}_E$, $\mathbf{p}_O$, $\mathbf{v}_E$, and $\mathbf{v}_O$, resulting in uncertain relative position and velocity.
Previous work incorporates bounded state uncertainty in \acp{VO} by inflating them using ellipsoidal uncertainty sets \cite{ellingson2020, ellingson2025}.
For position uncertainty, this inflation is computed exactly as the Minkowski sum of the nominal collision region and the position uncertainty set. 
However, velocity uncertainty causes the \ac{VO} to lose its conic structure, which \cite{ellingson2020} addresses by using cone boundaries tangent to the velocity uncertainty ellipse for over-approximating the velocity-dependent collision set.  
In contrast, we represent position and velocity uncertainty using general convex sets, incorporate position uncertainty through a Minkowski inflation of the collision region, and incorporate velocity uncertainty through an exact Minkowski sum with the \ac{VO}.
This formulation avoids conic approximations, preserving exact set geometry.

\begin{figure*}[b]
  \centering
  \resizebox{.75\linewidth}{!}{
  \usetikzlibrary{arrows.meta, calc}

\definecolor{colorA}{HTML}{16A085}      
\definecolor{colorB}{HTML}{8E44AD}      
\definecolor{colorRisk}{HTML}{CE1256}   

\newcommand{\agentsymbol}[5]{%
  \begin{scope}[shift={(#1,#2)}, rotate=#3]
    \fill[#4, draw=#5, line width=0.6pt]
      (0, 0.20) -- (-0.12, -0.12) -- (0, -0.04) -- (0.12, -0.12) -- cycle;
  \end{scope}
}
\newcommand{\agentsymbolcircle}[7]{%
  \draw[#6, line width=1.0pt] (#1,#2) circle (#7);
  \agentsymbol{#1}{#2}{#3}{#4}{#5}
}

\def\Wscale{0.03}   
\def\Vscale{0.15}   
\def\Yfac{0.35}      
\def\rEgo{10}       
\def\rObs{10}       

\newcommand{\ship}[8]{%
  \pgfmathsetmacro{\sx}{#1*\Wscale}
  \pgfmathsetmacro{\sy}{#2*\Wscale*\Yfac}
  \pgfmathsetmacro{\svx}{#3*\Vscale}
  \pgfmathsetmacro{\svy}{#4*\Vscale}
  \pgfmathsetmacro{\sn}{sqrt(\svx*\svx+\svy*\svy)}
  \pgfmathsetmacro{\sang}{atan2(#4,#3)-90}
  \pgfmathsetmacro{\srad}{#7*\Wscale}
  \draw[->, #5, line width=1.2pt] (\sx,\sy) -- ++(\svx,\svy);
  \node[#5, font=\scriptsize] at
    ({\sx+0.75*\svx-#8*0.3*\svy/\sn}, {\sy+0.75*\svy+#8*0.3*\svx/\sn}) {$\mathbf{v}_{#6}$};
  \agentsymbolcircle{\sx}{\sy}{\sang}{#5}{#5}{#5}{\srad}
  \node[#5, font=\scriptsize] at
    ({\sx+(\srad+0.22)*\svy/\sn}, {\sy-(\srad+0.22)*\svx/\sn}){};
}

\newcommand{\scenario}[9]{%
  \begin{scope}[shift={(#1,0)}]
    \pgfmathsetmacro{\losx}{#2*\Wscale}
    \pgfmathsetmacro{\losy}{#3*\Wscale*\Yfac}
    \pgfmathsetmacro{\losn}{sqrt(\losx*\losx+\losy*\losy)}
    \pgfmathsetmacro{\Rrange}{veclen(#2,#3)}
    \draw[colorRisk, line width=0.9pt, dash pattern=on 4pt off 3pt]
      (0,0) -- (\losx,\losy);
    \node[colorRisk, font=\scriptsize] at
      ({\losx/2+0.55*\losy/\losn}, {\losy/2-0.55*\losx/\losn}){};

    \ship{#2}{#3}{#4}{#5}{colorB}{O}{\rObs}{1}
    \ship{0}{0}{#6}{#7}{colorA}{E}{\rEgo}{#9}

    \node[gray!60, font=\small\bfseries] at ({\losx/2}, -1.3) {#8};
  \end{scope}
}

\begin{tikzpicture}[
  >=Stealth,
  font=\small,
  line cap=round,
  line join=round
]

\pgfmathsetmacro{\ytop}{6.0*\Yfac+0.9}   


\begin{scope}[shift={(-0.5, 1.5)}]   
  \draw[gray!60, rounded corners=4pt, line width=1pt]
    (0, 0) rectangle (2.7, 1.55);

  \agentsymbol{0.35}{1.18}{0}{colorA}{colorA}
  \node[right, colorA] at (0.6, 1.18) {\textbf{E / ego}};

  \agentsymbol{0.35}{0.78}{0}{colorB}{colorB}
  \node[right, colorB] at (0.6, 0.78) {\textbf{O / obstacle}};

  \draw[colorRisk, line width=1.2pt] (0.23, 0.38) -- (0.47, 0.38);
  \node[right, colorRisk] at (0.6, 0.38) {line of sight};
\end{scope}


\draw[gray!50, line width=1pt] (4.4,  -1.7) -- (4.4,  \ytop);
\draw[gray!50, line width=1pt] (7.7,  -1.7) -- (7.7,  \ytop);
\draw[gray!50, line width=1pt] (10.3, -1.7) -- (10.3, \ytop);
\draw[gray!50, line width=1pt] (13.4, -1.7) -- (13.4, \ytop);

\scenario{ 6.0}{  0}{200}{ 0}{-5}{0}{ 5}{Head-on}{1}
\scenario{ 0.0}{100}{100}{ 0}{-5}{5}{ 0}{Crossing}{-1}
\scenario{ 9.0}{  0}{120}{ 0}{ 4}{0}{ 8}{Overtaking}{1}
\scenario{12.0}{  0}{200}{ 0}{-5}{0}{-5}{Collinear}{1}
\scenario{15.0}{ 50}{  0}{ 0}{ 5}{0}{ 5}{Parallel}{1}

\pgfresetboundingbox
\path (1.0, -1.8) rectangle (16.5, {\ytop+0.2});
\end{tikzpicture}
  }
  \caption{Investigated encounters for evaluating collision risk metrics in 2D surface navigation.}
  \vspace*{-4mm}
  \label{fig:scenarios}
\end{figure*}
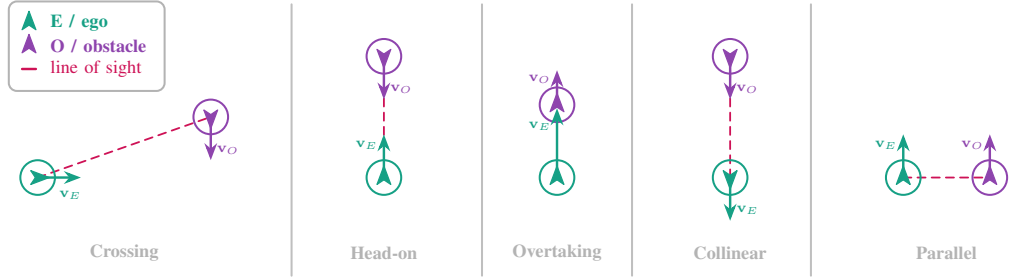

\begin{definition}[Uncertainty-Aware Velocity Obstacle (UVO)]\label{def:inflation}
    Let $\prelhat, \vrelhat$ denote the estimated relative position and velocity, respectively.
    Suppose the true values satisfy
    $\prel \in \prelhat \oplus \mathcal{K}_p$ and $\vrel \in \vrelhat \oplus \mathcal{K}_{v}$, where $\mathcal{K}_p$ and $\mathcal{K}_{v}$ are convex sets.
    Then, the \ac{UVO} is the set of ego velocities that result in a collision during the infinite prediction horizon:
    \begin{align}\label{eq:robust-vo}
        \UVO_{E\mid O} = &\Big\{\mathbf{v}_E \mid \vrelhat \in \{\mathbf{v} \mid\exists t\geq 0, \\
        &\prelhat + \mathbf{v} t \in \mathcal{E}\oplus (-\mathcal{O})\oplus (-\mathcal{K}_p)\} \oplus (-\mathcal{K}_{v})\Big\}. \notag
    \end{align}
\end{definition}

Constructing \acp{UVO} for convex sets requires repeated Minkowski sums. Rather than computing these sums explicitly, we represent each set by its support function \cite{Boyd_Vandenberghe_2004, girard2008}. 
The support function $h_\mathcal{S}: \mathbb{R}^n \to \mathbb{R}$ of a compact convex set $\mathcal{S} \subset \mathbb{R}^n$ is determined by the position of the supporting hyperplane in a given direction $\mathbf{d} \in \mathbb{R}^n$, $h_\mathcal{S}(\mathbf{d}) = \max_{\mathbf{x} \in \mathcal{S}}\mathbf{d} \cdot \mathbf{x}$~\cite{althoff2021}.
A corresponding support vector $v_{\mathcal{S}}(\mathbf{d}) \in \mathcal{S}$ is a boundary point of $\mathcal{S}$ that maximizes the dot product along $\mathbf{d}$: $\mathbf{d} \cdot v_{\mathcal{S}}(\mathbf{d}) = h_{\mathcal{S}}(\mathbf{d})$ \cite{girard2008}. 
Point containment of $\mathbf{x}' \in \mathbb{R}^n$ in $\mathcal{S}$ can then be verified by checking $\mathbf{d} \cdot \mathbf{x}' \leq \mathbf{d} \cdot v_{\mathcal{S}}(\mathbf{d}), \forall \mathbf{d} \in \mathbb{R}^n.$
Importantly, Minkowski addition reduces to point-wise addition under this representation: $h_{\mathcal{S}_1 \oplus \mathcal{S}_2}(\mathbf{d}) = h_{\mathcal{S}_1}(\mathbf{d}) + h_{\mathcal{S}_2}(\mathbf{d})$. 
Closed-form solutions exist for unit balls and ellipsoids, while polytopes require solving a linear program. The runtime complexity of these operations across common convex set representations is summarized in Table~\ref{tab:support}. 
While we consider general convex uncertainty sets, ellipsoids are of particular interest because they provide a natural representation of Gaussian uncertainty.

\subsection{Uncertainty-Aware \ac{CPA} Metrics}
Given the uncertainty sets used to obtain the \ac{UVO}, \ac{TCPA} and \ac{DCPA} can similarly be extended to account for relative state uncertainty. In particular, we formulate the following optimization problem over convex uncertainty sets $\mathcal{K}_p$ and $\mathcal{K}_{v}$ to bound a given \ac{CPA} metric:
\begin{align}\label{eq:CPA_uncertainty_opt}
& \underset{\prel, \vrel}{\min}\; \CPA \text{ s.t. } \prel \in \hat{\mathbf{p}}_{rel}\oplus\mathcal{K}_p, \vrel \in \hat{\mathbf{v}}_{rel}\oplus \mathcal{K}_v.
\end{align}
Replacing $\CPA$ with $\TCPA$ or $\DCPA$ yields the lower bound, while replacing it with  $- \TCPA$ or $- \DCPA$ yields the corresponding upper bound. Thus, solving each pair of optimization problems provides an interval containing possible \ac{CPA} metric values under the given uncertainty. 
The resulting optimization problems are generally nonconvex with respect to relative position and velocity.
Therefore, their computational cost depends on the chosen solution technique (see Sec. \ref{sec:experiments}).

In contrast to the deterministic case, the independently computed uncertainty intervals for \ac{TCPA} and \ac{DCPA} do not directly allow evaluation of Corollary~\ref{cor:vo-cpa}. The conjunction in Eq.~\eqref{eq:vointcpadcpa} depends on the joint relationship between the \ac{CPA} measures, which is not captured by their independently computed intervals in Eq.~\eqref{eq:CPA_uncertainty_opt}. 

\section{Numerical Experiments} \label{sec:experiments}
Assuming perfect state information, we numerically validate the equivalence between the \ac{CPA} metrics and \ac{VO} derived in Sec.~\ref{sec:theory}. We then introduce bounded uncertainty in relative position and velocity across a range of encounters in 2D surface navigation, as illustrated in Fig.~\ref{fig:scenarios}, and investigate the effects of this uncertainty on the individual metrics and their equivalence. We evaluate collision risk using a point-wise estimate given by the fraction of sampled points satisfying the collision condition associated with each metric.~\footnote{For reproducibility, we provide our implementation at: \href{https://github.com/eadietri/Collision_Risk_Metrics.git}{github.com/eadietri/Collision\_Risk\_Metrics.git}}.

\begin{figure*}[t]
\vspace{-2mm}
  \centering
  \resizebox{.75\linewidth}{!}{
  \input{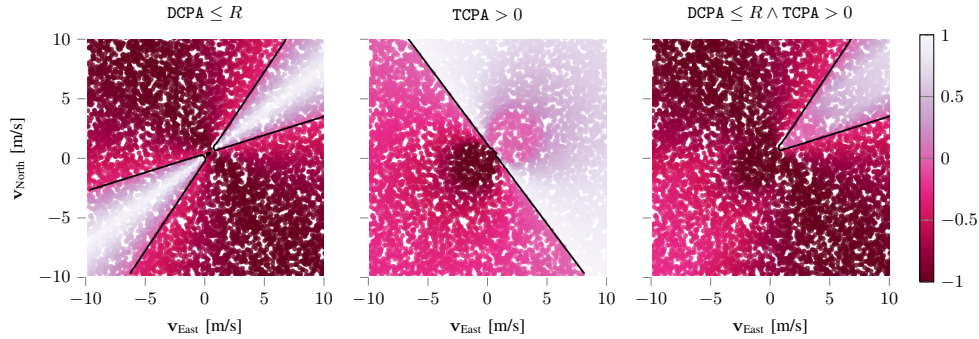}
  }
  \caption{
  Evaluation of \ac{CPA} metrics for a crossing encounter with sampled relative velocities in $[-10,10]^2$. DCPA defines two cones, while TCPA identifies the region of possible future collisions. Their conjunction, evaluated using quantitative semantics, recovers the VO in Corollary~\ref{cor:vo-cpa}. Color indicates the normalized magnitude of each metric, signed positive for satisfaction and negative for violation, scaled to $[0,1]$ using the 95th percentile of each metric. }
  \label{fig:equivalence}
  \vspace{-3mm}
\end{figure*}

\textbf{Equivalence:} To validate Theorem~\ref{thm:diskvo-cpa}, and specifically its characterization for initially separated agents in Corollary~\ref{cor:vo-cpa}, we visualize the individual \ac{CPA} metrics and their conjunction. We uniformly sample ego velocities $\mathbf{v}_{E}$ over the velocity space shown in Fig.~\ref{fig:equivalence}, compute the corresponding \ac{DCPA} and \ac{TCPA}, and shade each point according to the normalized signed value of each metric.
As shown in Fig.~\ref{fig:equivalence}, the conjunction of the \ac{CPA} conditions recovers the \ac{VO}.

\textbf{Bounded Uncertainty:} We consider a two-agent crossing encounter with relative position $\prel$, ego velocity $\mathbf{v}_{E}$, and obstacle velocity $\mathbf{v}_{O}$ bounded by elliptical uncertainty sets $\mathcal{K}_P, \mathcal{K}_E, \mathcal{K}_O$, respectively, centered on their nominal values. 
The relative velocity uncertainty set is obtained through point-wise subtraction of $\mathcal{K}_O$ and $\mathcal{K}_E$, and represented by an axis-aligned interval $\mathcal{K}_v$. 

\addtolength{\textheight}{-2mm}  

We first compute interval bounds on \ac{TCPA} and \ac{DCPA} by solving Eq.~\eqref{eq:CPA_uncertainty_opt} using SciPy's SLSQP solver.
The ellipsoidal sets are characterized by support functions, resulting in a fixed number of decision variables in two dimensions. Consequently, each optimization iteration has $\mathcal{O}(1)$ computational cost, yielding an overall complexity of $\mathcal{O}(K)$ for $K$ iterations. 
We uniformly sample $10{,}000$ pairs of positions and velocities from the uncertainty sets and compute the corresponding \ac{TCPA} and \ac{DCPA}.
Fig.~\ref{fig:uniform} shows their joint distribution, with samples satisfying Corollary~\ref{cor:vo-cpa}'s collision condition highlighted in pink, and the resulting \ac{TCPA} and \ac{DCPA} intervals from Eq.~\eqref{eq:CPA_uncertainty_opt} shown in green.
We then compare the collision risk identified by the \ac{CPA} metrics with that obtained from the deterministic and uncertainty-aware \acp{VO}. Fig.~\ref{fig:uniform} depicts both \acp{VO} in velocity space. The uncertainty set of the ego velocity $\mathcal{K}_E$ lies entirely outside the deterministic \ac{VO}, but partially overlaps with the \ac{UVO}. 
The resulting collision risks are $0.192$ for the \ac{CPA}-based condition of Corollary~\ref{cor:vo-cpa}, $0.000$ for the deterministic \ac{VO}, and $0.531$ for the \ac{UVO}. These results demonstrate that independently bounded \ac{CPA} metrics are insufficient for characterizing collision risk under uncertainty, whereas the \ac{UVO} preserves characterization of worst-case collision risk.

\begin{figure}[b]
\centering
\resizebox{\linewidth}{!}{

\begin{tikzpicture}
\begin{groupplot}[
    group style={
        group size=2 by 1,
        horizontal sep=1.7cm,
    },
    width=\columnwidth,
    height=\columnwidth,
    title={},
    xlabel style={font=\Large},
    title style={font=\large},
    ticklabel style={font=\large},
    tick align=outside,
    xtick pos=left,
    ytick pos=left,
]

\nextgroupplot[
    xmin=2.0386, xmax=47.9284,
    ymin=-100, ymax=194.0223,
    xlabel={$\TCPA [s]$},
    ylabel={$\DCPA [m]$},
    ylabel style={font=\Large, yshift=-0.5cm},
]
  \addplot graphics [
    xmin=2.0386, xmax=47.9284,
    ymin=-194.0223, ymax=194.0223,
  ] {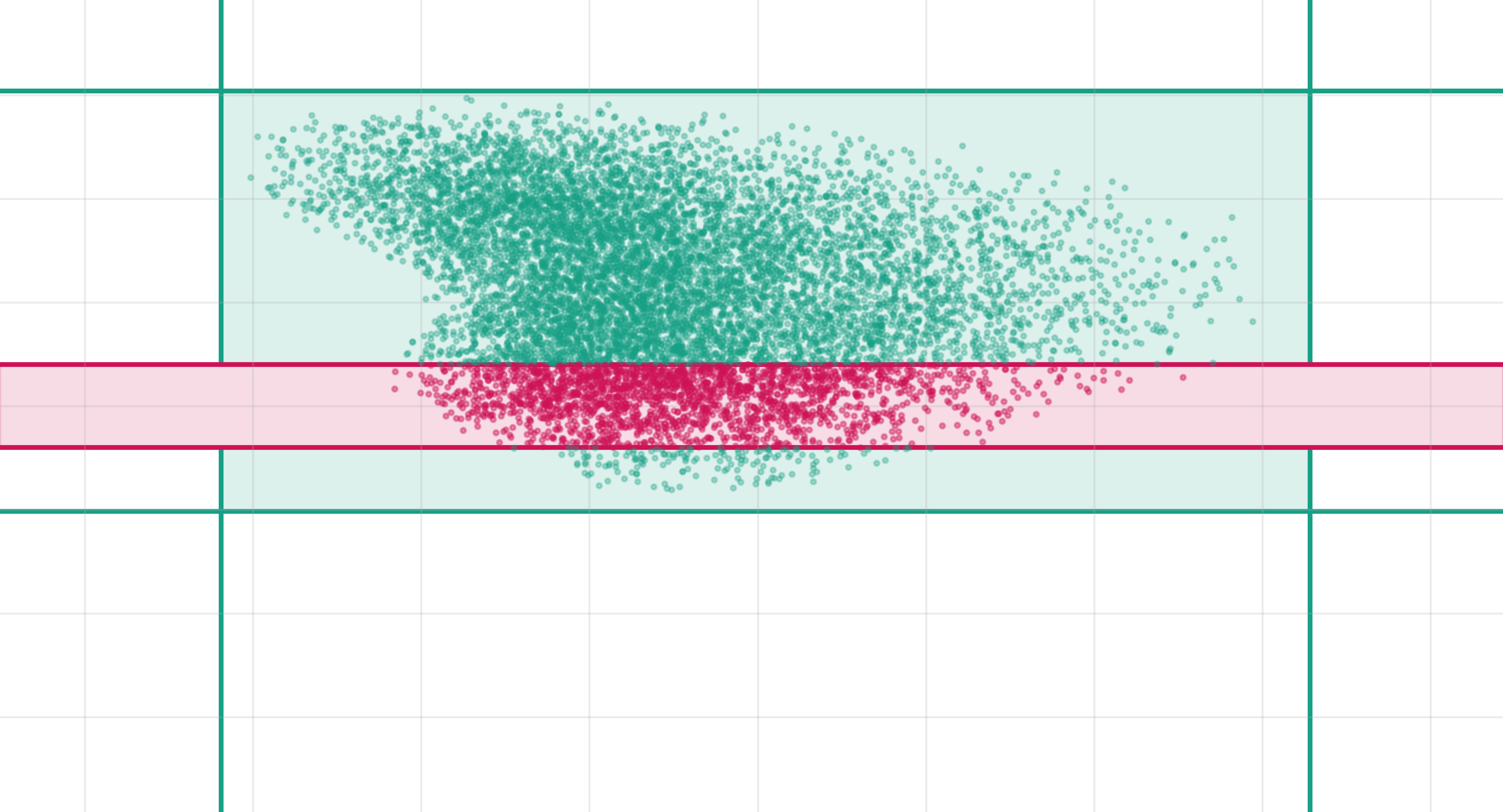};

\nextgroupplot[
    xmin=-3.8, xmax=7,
    ymin=-7.75, ymax=3.5,
    axis equal image,
    xlabel={$\mathbf{v}_{\text{East}}[m/s]$},
    ylabel={$\mathbf{v}_{\text{North}}[m/s]$},
    ylabel style={font=\Large, yshift=-0.2cm},
]
  \addplot graphics [
    xmin=-4.2809, xmax=7.7870,
    ymin=-8.1347, ymax=3.9332,
  ] {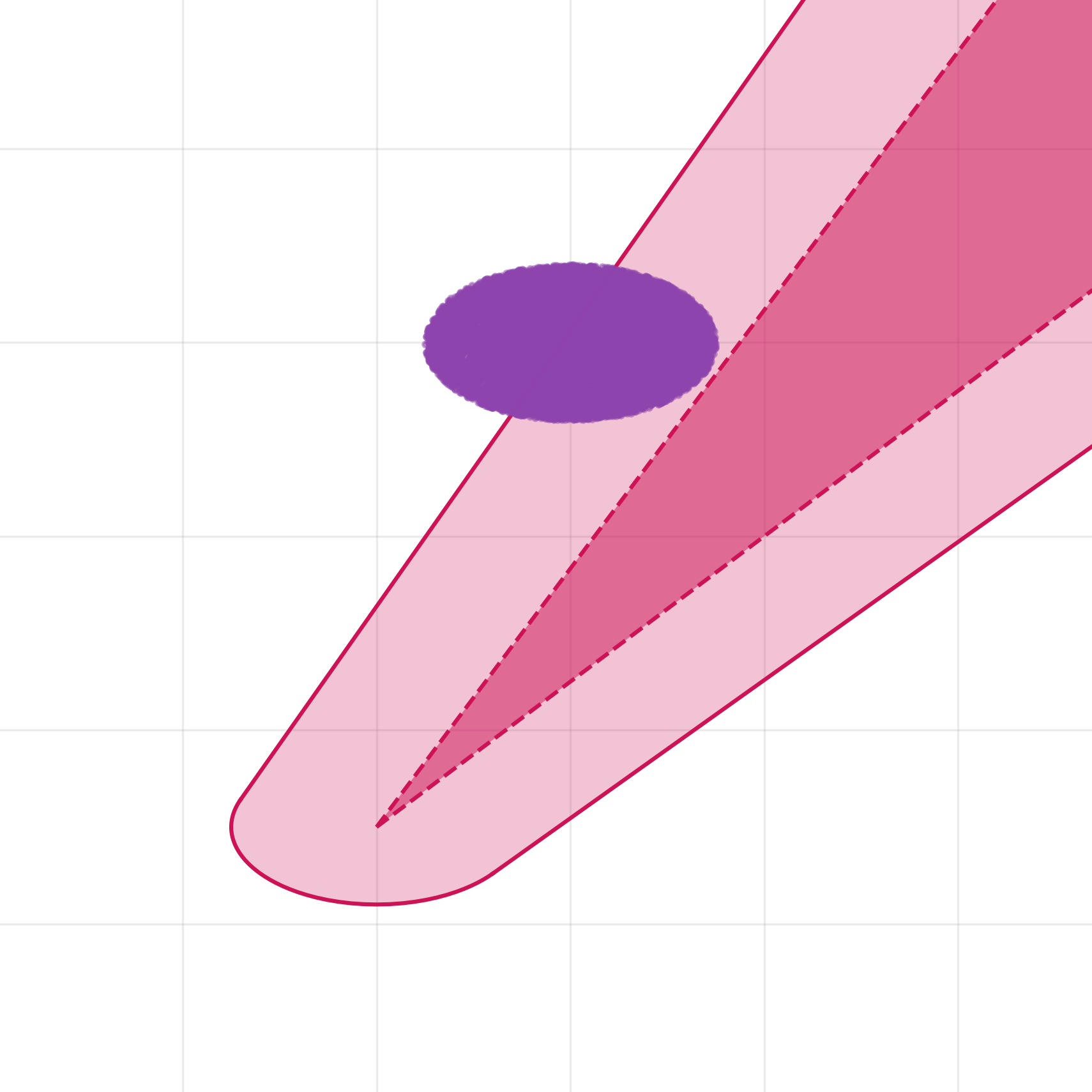};

\end{groupplot}
\end{tikzpicture}
    }
  \caption{Collision risk metrics under bounded state uncertainty for crossing encounter: (Left) Joint distribution of $10{,}000$ TCPA and DCPA values, with collision condition of Corollary~\ref{cor:vo-cpa} (pink) and interval bounds from Eq.~\eqref{eq:CPA_uncertainty_opt} (green). (Right) Deterministic VO (dark pink) and UVO (light pink) in velocity space, with the ego-velocity uncertainty set $\mathcal{K}_E$ (purple).}
  \label{fig:uniform}
\end{figure}

\textbf{Estimated Collision Risk:} We compare the collision risk identified by the \ac{CPA} metrics and \acp{VO} across five distinct agent encounters, illustrated in Fig.~\ref{fig:scenarios}, using Monte Carlo sampling. The encounters use the bounded position and velocity uncertainty formulation introduced above.
For each encounter, we uniformly sample $10{,}000$ pairs of positions and velocities from the uncertainty sets, compute the corresponding \ac{TCPA} and \ac{DCPA}, and construct the deterministic and uncertainty-aware \acp{VO}. The empirical collision risk for each metric is reported in Table~\ref{tab:accuracy}. The \ac{UVO} yields a collision risk of approximately 1.000 across all encounters, consistent with its worst-case formulation. The \ac{UVO} encompasses the union of all velocity obstacles over the specified uncertainty bounds, and therefore classifies an encounter as a collision whenever a collision is possible under any realization of the uncertainty.

\section{Conclusion}
The \ac{CPA} metrics and \ac{VO} provide related representations of collision risk. Our equivalence results show that, in the absence of uncertainty and under the stated assumptions, the geometry of the \ac{VO} encodes the \ac{CPA}-based risk information for all candidate trajectories without loss of information. 
However, the \ac{VO} and \ac{CPA} metrics differ in how they accommodate uncertainty. For the \ac{VO}, convex uncertainty sets in relative position and velocity can be incorporated directly through Minkowski sums, which can be evaluated efficiently using support functions. 
In contrast, the \ac{CPA} metrics 
require both the \ac{TCPA} and \ac{DCPA} to provide a \ac{UVO}-related assessment of collision risk. Computing the extrema of these measures over uncertainty sets requires solving nonconvex optimization problems, and independent optimization of the two metrics can yield extrema attained at different states within the set.
Consequently, the conditions for \ac{VO} membership expressed in terms of \ac{CPA} metrics do not generally hold under uncertainty, as confirmed by our empirical results.
Future work should assess how these risk metrics can be used to design safe controllers under uncertainty, extend the analysis to a probabilistic setting with unbounded disturbances, and investigate the uncertainty-aware \acp{UVO} across finite time horizons. 

\begin{table}[t]
  \caption{Point Estimates of Collision Risk Identified by the CPA metrics, VO, and UVO
  }
  \label{tab:accuracy}
  \centering
  \small
  \setlength{\tabcolsep}{6pt}
  \renewcommand{\arraystretch}{1.15}
  \sisetup{table-format=1.3, detect-weight=true, mode=text}
  \begin{tabular}{l *{5}{S}}
    \toprule
    {\textbf{Metric}} & {\textbf{Cr}} & {\textbf{H}} & {\textbf{O}} & {\textbf{Co}} & {\textbf{P}} \\
    \midrule
    $\UVO$               & 1.000 & 1.000 & 1.000 & 0.999 & 1.000 \\
    $\VO$                & 0.917 & 0.776 & 0.554 & 0.017 & 0.216 \\
    $\DCPA \wedge \TCPA$   & 0.699 & 0.627 & 0.519 & 0.019 & 0.212 \\
    $\DCPA$              & 0.699 & 0.627 & 0.519 & 0.036 & 0.428 \\
    $\TCPA$              & 1.000 & 1.000 & 1.000 & 0.501 & 0.505 \\
    \bottomrule
  \end{tabular}
  \begin{center}
      \footnotesize{Cr: Crossing, H: Head-on, O: Overtaking, Co: Collinear, P: Parallel}
  \end{center}
  \vspace{-4mm}
\end{table}




\bibliography{literature}
\bibliographystyle{IEEEtran}

\end{document}